%% file: main.tex
\documentclass[letterpaper,10pt,conference]{ieeeconf}
\IEEEoverridecommandlockouts
\usepackage{cite}

\usepackage{amsmath}

\usepackage{amsthm, amssymb, verbatim}
\usepackage{array, booktabs, graphicx, multirow}
\usepackage[table]{xcolor}
\usepackage{capt-of}
\usepackage{stfloats}
\theoremstyle{definition}

\theoremstyle{plain}

\newtheorem{proposition}{Proposition}[section]

\theoremstyle{remark}

\newcommand{\methodname}{CATS}

\begin{document}
\title{\LARGE \bf Revisiting Temporal Regularization for Smooth Control 
	
	in Deep Reinforcement Learning}

\author{SungJae Ahn$^{1}$, Jeong Woon Lee$^{1}$, Kyoleen Kwak$^{1}$, and Hyoseok Hwang$^{1,\dagger}$%
\thanks{$^{1}$College of Software, Kyung Hee University.}%
\thanks{$^{\dagger}$Corresponding author: \texttt{hyoseok@khu.ac.kr}.}%
}

\maketitle
\thispagestyle{empty}
\pagestyle{empty}

\begin{abstract}
Deep Reinforcement Learning policies can produce nonsmooth action oscillations
that hinder deployment on physical robots. Existing architectural and
penalty-based approaches seek spatial smoothness by directly reducing sensitivity to
changes in state inputs, but their broad constraints can degrade task performance
as stronger smoothing is pursued. Temporal regularization instead constrains
action differences along observed transitions, but has been considered unable
to provide the spatial smoothness needed under observation noise. We revisit
this assumption by proving that the temporal penalty bounds the expected action differences
between current states sharing a next state, revealing a spatial effect that
empirically extends to spatial smoothness. Building on this finding, we propose
Conditioning for Action using only Temporal Smoothness (\methodname{}), which
combines a temporal penalty with linear ramp-up. We highlight
temporal regularization's ability to provide spatial smoothness while better
preserving task performance than explicit spatial regularization.
Through linear ramp-up, \methodname{} allows the policy
to learn rewarding behavior before progressively smoothing its actions,
improving return preservation and both temporal and spatial smoothness.
Experiments in both simulation and the real world show that
\methodname{} substantially reduces action oscillation without degrading task
performance, with little computational overhead.
\end{abstract}

\section{Introduction}
Deep Reinforcement Learning (RL) has emerged as a powerful
framework for continuous control~\cite{lillicrap2016continuous},
enabling robots to acquire complex motor skills through interaction
with their environment~\cite{andrychowicz2020learning}.
However, due to optimization instability in deep RL~\cite{lee2026stabilizing} and observation noise~\cite{mysore2021regularizing,song2025lipsnetpp}, policies for robotic control can produce nonsmooth action oscillations that hinder their deployment on physical systems~\cite{christmann2024benchmarking,azimi2026benchmarking}.
Such oscillations not only compromise control reliability~\cite{kumar2021rma}, but also accelerate mechanical wear, jeopardize hardware safety, and increase energy consumption~\cite{mysore2021regularizing,chen2025learning}.

Reducing action oscillations in real-world deployment
calls for two policy properties, \textit{temporal smoothness} and
\textit{spatial smoothness}~\cite{mysore2021regularizing}.
Temporal smoothness refers to low policy sensitivity
to state inputs received over time and directly relates to action
oscillation in simulation. Spatial smoothness refers to low policy sensitivity to small changes
in state inputs, thus reducing sensitivity to real-world noise.
To promote either or both, prior work has
employed corresponding regularization methods.

Although action oscillation is fundamentally a change
in policy outputs over time, many existing methods view action smoothness as a property
of the policy function without considering the temporal ordering of its inputs.
Accordingly, these spatial regularization methods seek spatial smoothness
directly by encouraging actions to vary less across any local state changes.
While these approaches can reduce action oscillation
and noise sensitivity, they act across all state-space directions and may
therefore restrict action changes needed for high return.

Temporal regularization, in contrast, exploits the temporal ordering
of policy inputs to reduce action variation over time. Because it constrains the policy only along
transition directions that are actually observed, it can pursue stronger
smoothness with less degradation in task performance.
This same focus, however, appears to leave other directions
unconstrained, so even a policy that is smooth along its trajectories may remain
sensitive to observation noise in those directions. Temporal
regularization has thus been regarded as unrelated to spatial smoothness and
has commonly been paired with a spatial
penalty~\cite{mysore2021regularizing,cao2023image,kwak2026enhancing,lee2026stabilizing},
reintroducing broad state-space constraints and leaving the relative
weighting of the two penalties unclear.

The complexity introduced by existing smoothness
methods is another, often underappreciated challenge. While these methods aim
to facilitate real-world deployment by
reducing action oscillation, this complexity can hinder the very deployment
they target. Robotic RL
already involves many interacting hyperparameters and demanding training
pipelines, while evaluations are expensive and performance is sensitive to
configuration choices~\cite{mahmood2018benchmarking,eimer2023hyperparameters}.
Expanding the hyperparameter space or increasing training time
therefore makes smoothness methods harder to apply in practice.

\begin{figure}[t]
    \centering
    \includegraphics[width=\columnwidth]{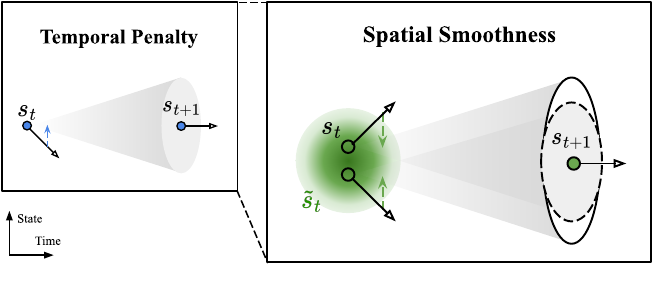}
    \vspace{-10pt}
    \caption{From temporal regularization to spatial
    smoothness. Black arrows denote actions; colored
    arrows indicate alignment. The temporal penalty aligns actions along
    transitions (left). When nearby states share possible next states, this
    alignment can indirectly bring their actions closer (right).}
    \label{fig:transition_conditioned_smoothing}
\vspace{-10pt}
\end{figure}

In this work, we aim to develop a practical approach
that achieves both temporal and spatial smoothness while preserving task
performance. To this end, we
revisit the assumption that temporal regularization cannot
provide the spatial smoothness needed under observation noise.
During training, the temporal penalty acts on next
states drawn from a distribution rather than a single fixed next state. When two current states share a next state,
their actions are encouraged to approach the same next-state action,
indirectly becoming similar.
We establish this spatial effect by proving that the temporal penalty bounds
the expected action difference between these states. We further hypothesize that this effect
extends to spatial smoothness when nearby states share
more possible next states (Fig.~\ref{fig:transition_conditioned_smoothing}). Our experiments confirm this effect and the
resulting spatial smoothness, showing that a temporal penalty alone is
sufficient for both forms of smoothness in the evaluated tasks.

Building on this finding, we propose Conditioning for
Action using \textit{only }Temporal Smoothness (\textbf{\methodname{}}), a lightweight
regularizer that applies a temporal penalty with a
linearly increasing coefficient (linear ramp-up). We highlight
temporal regularization's ability to provide spatial smoothness while preserving
performance, unlike explicit spatial regularization methods.
Linear ramp-up also strengthens these benefits by allowing the policy
to learn rewarding behavior before smoothing its actions, rather than seeking
rewarding behavior under smoothness constraints.
\methodname{} uses no spatial penalty and adds only one
hyperparameter, with little computational overhead and minimal implementation
changes.

Extensive experiments show that \methodname{} substantially reduces action
oscillation without decreasing return, particularly on higher-dimensional tasks,
and continues to do so under observation noise. We
also directly examine the spatial smoothness induced by \methodname{}, compare it with
that produced by explicit spatial regularization, and introduce a simple
extension that broadens the induced spatial effect. Finally, sim-to-real
experiments show that \methodname{} effectively reduces action oscillation
on a physical robot despite its simple design.

\section{Related Work}

\subsection{Spatial Regularization}

Architectural approaches formulate smoothness solely as a property of
the policy's input-to-output mapping. These methods
often promote spatial smoothness by reducing the policy's Lipschitz constant~\cite{takase2022stability}
without assumptions about the input states.
LipsNet~\cite{song2023lipsnet} enforces this constraint in its forward pass,
requiring Jacobian computation at inference. LCP~\cite{chen2025learning}
and LipsNet++~\cite{song2025lipsnetpp} instead constrain
the policy's Lipschitz constant through computationally expensive gradient-based
penalties during training.

Penalty-based spatial regularization methods seek
spatial smoothness by explicitly penalizing action differences between nearby
states. CAPS~\cite{mysore2021regularizing} draws
Gaussian-perturbed states around the current state.
L2C2~\cite{kobayashi2022l2c2} narrows the neighborhood using the
displacement between consecutive states. ASAP~\cite{kwak2026enhancing} aligns the current action with an expected action prediction using
transitions from the previous state. SR$^2$L~\cite{shen2020deep} instead searches for perturbations
through iterative gradient ascent during each policy
update.

\subsection{Temporal Regularization}

Temporal regularization promotes temporal
smoothness by penalizing policy-output differences along sampled next-state
directions, as in CAPS~\cite{mysore2021regularizing}.
Grad-CAPS~\cite{lee2024gradient} instead penalizes changes in action differences,
allowing actions to vary while changing smoothly.
However, it requires two consecutive transitions, adding computation and
requiring buffer modifications to retain the additional state.
Because temporal regularization has been regarded as unrelated to spatial
smoothness, prior methods have combined it with a separate spatial
penalty~\cite{mysore2021regularizing,cao2023image,kwak2026enhancing,lee2026stabilizing}.
These additional penalties introduce hyperparameters and computational overhead, with no clear
criterion for balancing the two penalties.

Reward penalties are a widely used form of temporal
regularization that penalizes action changes over time, but prior work has
reported that they can instead increase action
oscillation under sparse rewards~\cite{chen2021addressing,song2023lipsnet,song2025lipsnetpp}.
Other approaches exploit temporal information
through policy architectures~\cite{wang2024smooth,ICLR2025_6e0274fd}
or action-selection mechanisms~\cite{chen2021addressing,yu2021taac}.
Compared with loss-based regularization, these
approaches can increase implementation complexity, hyperparameters, and computational cost.

\section{Background}

\subsection{Action Smoothness in Reinforcement Learning}
\label{subsec:rl_smoothness}

Throughout this paper, action smoothness\footnote{We also use ``smoothness'' as shorthand for action smoothness, distinct from the policy properties of temporal and spatial smoothness.} refers to the smoothness of
generated action sequences, regardless of observation noise in the
environment.
For observation $o$ and action $a$, let $\pi(a\mid o)$ denote the policy
interacting with the environment and $\mu(o)\in\mathbb R^{d_a}$ its
deterministic output or stochastic mean, where $d_a$ is the action dimension.
Following prior work, we define regularization using $\mu$ and evaluate the
resulting action sequence $a_t=\mu(o_t)$ after training. This excludes
action-sampling noise from the smoothness evaluation and reflects deterministic deployment behavior.
We use the smoothness score of Mysore et
al.~\cite{mysore2021regularizing}, which measures high-frequency content in the
action signal; lower scores indicate smoother actions.

\subsection{Temporal and Spatial Smoothness}
\label{subsec:temporal_spatial_smoothness}

We refer to the policy properties required for action smoothness as temporal
and spatial smoothness. Temporal smoothness concerns policy sensitivity to
inputs encountered over time, while spatial smoothness concerns its sensitivity
to small state-input changes in any direction~\cite{mysore2021regularizing,kwak2026enhancing}.

To illustrate the roles of these two properties, we express action variation
in continuous time.
Let $J_{\mu}(o)=\partial\mu(o)/\partial o\in\mathbb R^{d_a\times d_s}$
denote the policy Jacobian matrix, which describes how
the policy output changes with its input, where $d_s$ is the dimension of the state
and observation vectors.
When the policy directly observes the underlying state, $o_t=s_t$, the
action derivative is the instantaneous rate of change
of the policy output over time,
\begin{equation}
    \left\|\frac{d a_t}{dt}\right\|
    =
    \left\|J_{\mu}(s_t)\frac{d s_t}{dt}\right\|.
    \label{eq:noise_free_temporal_direction}
\end{equation}
Therefore, in the noise-free case, temporal smoothness
is the policy property directly related to action oscillation.

To make the role of spatial smoothness explicit, as in
Song et al.~\cite{song2025lipsnetpp}, we assume that the actor observes
$o_t=s_t+\epsilon_t$ during deployment, where $\epsilon_t$ denotes the observation
noise at time step $t$. The action derivative then decomposes as
\begin{equation}
    \left\|\frac{d a_t}{dt}\right\|
    =
    \Biggl\|
    J_{\mu}(o_t)
    \underbrace{\frac{d s_t}{dt}}_{\text{state direction}}
    +
    J_{\mu}(o_t)
    \underbrace{\frac{d\epsilon_t}{dt}}_{\text{noise direction}}
    \Biggr\|.
    \label{eq:state_noise_exact_decomposition}
\end{equation}
The noise direction in the second component is generally unknown and
uncontrollable. However, its contribution can be bounded by the product of
the policy Jacobian norm and the noise rate, yielding
\begin{equation}
    \left\|\frac{d a_t}{dt}\right\|
    \le
    \underbrace{
    \left\|J_{\mu}(o_t)\frac{d s_t}{dt}\right\|
    }_{\text{temporal smoothness}}
    +
    \underbrace{
    \left\|J_{\mu}(o_t)\right\|
    }_{\text{spatial smoothness}}
    \left\|\frac{d\epsilon_t}{dt}\right\|.
    \label{eq:state_noise_directional_bound}
\end{equation}
Therefore, under observation noise, spatial smoothness is also
needed to limit action variation along noise directions.
Methods combining temporal and spatial regularization seek the
corresponding policy properties by reducing the first and second terms in
\eqref{eq:state_noise_directional_bound}, respectively.

Alternatively, spatial smoothness alone can provide
temporal smoothness, since its low sensitivity across state-space directions
also applies along the state trajectory. Applying the relaxation to the first
term in \eqref{eq:state_noise_directional_bound} gives
    the looser bound~\cite{song2025lipsnetpp}
\begin{equation}
    \left\|\frac{d a_t}{dt}\right\|
    \le
    \underbrace{
    \left\|J_{\mu}(o_t)\right\|
    }_{\text{spatial smoothness}}
    \left(\left\|\frac{d s_t}{dt}\right\|+
    \left\|\frac{d\epsilon_t}{dt}\right\|\right).
    \label{eq:full_jacobian_temporal_bound}
\end{equation}
Most architectural and spatial-only regularization methods seek to reduce the
bound in \eqref{eq:full_jacobian_temporal_bound}.
Specifically, LipsNet~\cite{song2023lipsnet} controls $\|J_{\mu}(o_t)\|$ at inference
time, whereas LipsNet++~\cite{song2025lipsnetpp} and
LCP~\cite{chen2025learning} constrain $\|J_{\mu}(s_t)\|$ during training.

\section{Revisiting Temporal Regularization}

\subsection{Scope of Spatial and Temporal Regularization}
\label{subsec:scope_regularization}

For a sampled transition, the CAPS temporal penalty~\cite{mysore2021regularizing}
measures the action difference along the direction $s_{t+1}-s_t$, with the local
first-order approximation
\begin{equation}
    \left\|\mu(s_{t+1})-\mu(s_t)\right\|
    \approx
    \left\|J_{\mu}(s_t)(s_{t+1}-s_t)\right\|.
    \label{eq:clean_temporal_direction}
\end{equation}

The key distinction between spatial and temporal regularization lies in the
scope of state-space directions they constrain. Let
$\hat v_t=(s_{t+1}-s_t)/\|s_{t+1}-s_t\|_2$ denote the normalized transition
direction. The quantities constrained by the two regularizers are
\begin{equation*}
\begin{aligned}
    \text{Spatial}\quad
    &\left\|J_{\mu}(s_t)v\right\|_2,
    \qquad v\in\mathbb S^{d_s-1},\\
    \text{Temporal}\quad
    &\left\|J_{\mu}(s_t)\hat v_t\right\|_2.
\end{aligned}
\end{equation*}
Here, $\mathbb S^{d_s-1}$ is the unit sphere in the state space.
Thus, temporal regularization constrains the policy Jacobian only along the
sampled next-state direction, whereas spatial regularization constrains it along
all local state-space directions.

In control tasks, retaining the current action at the next time step can still yield high return~\cite{chen2021addressing,yu2021taac}. An arbitrary
state displaced from $s_t$ by the same magnitude, however, has no such
transition relation, and enforcing the same action there may incur greater
return loss. Therefore, spatial regularization that does not
account for the temporal relations, which govern oscillation, may prevent the policy
from selecting rewarding actions. This motivates revisiting
temporal regularization as a way to pursue stronger action smoothness while better
preserving task performance, as examined in
Section~\ref{subsec:temporal_spatial_regularization}.

\subsection{Spatial Effect of the Temporal Penalty}
\label{subsec:dynamics_aware_temporal_penalty}

Temporal regularization has been regarded as unrelated
to spatial smoothness because it constrains only the next-state direction of
each transition. During training, however, the next state is sampled from a
distribution rather than fixed.
To examine the implication of this sampling, we consider the one-step
temporal penalty used in CAPS~\cite{mysore2021regularizing}.
Let $d_\pi$ denote the current-state distribution under $\pi$, and
let $P_\pi(\cdot\mid s)$ denote the next-state distribution jointly induced by
the policy and environment dynamics.
The expected temporal penalty is
\begin{equation}
    L_{\mathrm{T}}
    :=
    \mathbb E_{s\sim d_\pi}
    \mathbb E_{s'\sim P_\pi(\cdot\mid s)}
    \left[\|\mu(s)-\mu(s')\|_2^2\right].
    \label{eq:squared_temporal_penalty_mean_actor}
\end{equation}
Let $d_\pi^+$ denote the resulting next-state distribution, and let
$\rho_\pi(\cdot\mid s')$ denote the conditional
distribution over current states, given that the next state is $s'$.
We measure the induced spatial effect between two conditionally
independent current states sharing $s'$ as
\begin{equation}
    L_{\mathrm{S}}
    :=
    \mathbb E_{s'\sim d_\pi^+}
    \mathbb E_{u,v\overset{\mathrm{i.i.d.}}{\sim}\rho_\pi(\cdot\mid s')}
    \left[\|\mu(u)-\mu(v)\|_2^2\right].
    \label{eq:transition_conditioned_spatial_penalty}
\end{equation}
$L_{\mathrm S}$ measures the average squared action difference
between current states that share possible next states.

Under this conditioning, the temporal
penalty encourages both $\mu(u)$ and $\mu(v)$ to approach the common
next-state action $\mu(s')$, indirectly encouraging similar current-state
actions.

\begin{table*}[t]
\centering
\vspace*{4pt}
\caption{Clean-observation results.}
\label{tab:main_results}
\scriptsize
\setlength{\tabcolsep}{0.8pt}
\renewcommand{\arraystretch}{1.4}
\setlength{\aboverulesep}{0.25ex}
\setlength{\belowrulesep}{0.25ex}
\resizebox{\textwidth}{!}{
\begin{tabular}{l l >{\columncolor{gray!10}}l l >{\columncolor{gray!10}}l l >{\columncolor{gray!10}}l l >{\columncolor{gray!10}}l l >{\columncolor{gray!10}}l @{\hspace{5pt}}l >{\columncolor{gray!10}}l l >{\columncolor{gray!10}}l l >{\columncolor{gray!10}}l l >{\columncolor{gray!10}}l l >{\columncolor{gray!10}}l}
\toprule
\multirow{3}{*}{Method} & \multicolumn{10}{c}{PPO} & \multicolumn{10}{c}{SAC} \\
\cmidrule(lr){2-11} \cmidrule(lr){12-21}
& \multicolumn{2}{c}{Ant} & \multicolumn{2}{c}{Hopper} & \multicolumn{2}{c}{LunarLander} & \multicolumn{2}{c}{Pendulum} & \multicolumn{2}{c}{Walker}
& \multicolumn{2}{c}{Ant} & \multicolumn{2}{c}{Hopper} & \multicolumn{2}{c}{LunarLander} & \multicolumn{2}{c}{Pendulum} & \multicolumn{2}{c}{Walker} \\
\cmidrule(lr){2-3} \cmidrule(lr){4-5} \cmidrule(lr){6-7} \cmidrule(lr){8-9} \cmidrule(lr){10-11}
\cmidrule(lr){12-13} \cmidrule(lr){14-15} \cmidrule(lr){16-17} \cmidrule(lr){18-19} \cmidrule(lr){20-21}
& \multicolumn{1}{c}{$R \uparrow$} & \multicolumn{1}{c}{\cellcolor{gray!10}$sm \downarrow$} & \multicolumn{1}{c}{$R \uparrow$} & \multicolumn{1}{c}{\cellcolor{gray!10}$sm \downarrow$} & \multicolumn{1}{c}{$R \uparrow$} & \multicolumn{1}{c}{\cellcolor{gray!10}$sm \downarrow$} & \multicolumn{1}{c}{$R \uparrow$} & \multicolumn{1}{c}{\cellcolor{gray!10}$sm \downarrow$} & \multicolumn{1}{c}{$R \uparrow$} & \multicolumn{1}{c}{\cellcolor{gray!10}$sm \downarrow$}
& \multicolumn{1}{c}{$R \uparrow$} & \multicolumn{1}{c}{\cellcolor{gray!10}$sm \downarrow$} & \multicolumn{1}{c}{$R \uparrow$} & \multicolumn{1}{c}{\cellcolor{gray!10}$sm \downarrow$} & \multicolumn{1}{c}{$R \uparrow$} & \multicolumn{1}{c}{\cellcolor{gray!10}$sm \downarrow$} & \multicolumn{1}{c}{$R \uparrow$} & \multicolumn{1}{c}{\cellcolor{gray!10}$sm \downarrow$} & \multicolumn{1}{c}{$R \uparrow$} & \multicolumn{1}{c}{\cellcolor{gray!10}$sm \downarrow$} \\
\midrule
Base & 1411\,{\tiny (506)} & \cellcolor{gray!10}1.464\,{\tiny (0.193)} & 2658\,{\tiny (685)} & \cellcolor{gray!10}1.162\,{\tiny (0.265)} & \underline{221.2}\,{\tiny (28.6)} & \cellcolor{gray!10}0.209\,{\tiny (0.012)} & -217.7\,{\tiny (66.1)} & \cellcolor{gray!10}0.469\,{\tiny (0.031)} & 3510\,{\tiny (582)} & \cellcolor{gray!10}1.682\,{\tiny (0.237)} & 3656\,{\tiny (675)} & \cellcolor{gray!10}2.286\,{\tiny (0.213)} & 2656\,{\tiny (702)} & \cellcolor{gray!10}1.507\,{\tiny (0.273)} & \underline{275.9}\,{\tiny (6.1)} & \cellcolor{gray!10}0.424\,{\tiny (0.052)} & -149.3\,{\tiny (1.4)} & \cellcolor{gray!10}0.984\,{\tiny (0.103)} & 4189\,{\tiny (676)} & \cellcolor{gray!10}1.268\,{\tiny (0.126)} \\
\midrule
CAPS & 1873\,{\tiny (469)} & \cellcolor{gray!10}1.222\,{\tiny (0.170)} & \cellcolor{red!18}2440\,{\tiny (608)} & \cellcolor{gray!10}0.471\,{\tiny (0.049)} & \cellcolor{red!18}205.6\,{\tiny (22.2)} & \cellcolor{red!18}0.214\,{\tiny (0.014)} & \textbf{-175.2}\,{\tiny (5.4)} & \cellcolor{gray!10}\underline{0.421}\,{\tiny (0.015)} & 3626\,{\tiny (666)} & \cellcolor{gray!10}0.505\,{\tiny (0.056)} & \cellcolor{red!18}3482\,{\tiny (543)} & \cellcolor{gray!10}2.264\,{\tiny (0.116)} & 2856\,{\tiny (393)} & \cellcolor{gray!10}1.235\,{\tiny (0.087)} & \cellcolor{red!18}269.4\,{\tiny (11.5)} & \cellcolor{gray!10}0.355\,{\tiny (0.031)} & \cellcolor{red!18}-177.0\,{\tiny (54.0)} & \cellcolor{gray!10}\underline{0.376}\,{\tiny (0.036)} & \underline{4651}\,{\tiny (357)} & \cellcolor{gray!10}1.178\,{\tiny (0.101)} \\
\midrule
L2C2 & \cellcolor{red!18}1262\,{\tiny (358)} & \cellcolor{red!18}1.515\,{\tiny (0.157)} & 2704\,{\tiny (545)} & \cellcolor{red!18}1.502\,{\tiny (0.435)} & \cellcolor{red!18}208.2\,{\tiny (26.7)} & \cellcolor{red!18}0.211\,{\tiny (0.020)} & -184.8\,{\tiny (17.1)} & \cellcolor{gray!10}0.452\,{\tiny (0.019)} & \cellcolor{red!18}3342\,{\tiny (547)} & \cellcolor{red!18}1.744\,{\tiny (0.157)} & \cellcolor{red!18}2927\,{\tiny (379)} & \cellcolor{gray!10}2.276\,{\tiny (0.121)} & 2965\,{\tiny (614)} & \cellcolor{gray!10}1.249\,{\tiny (0.105)} & \cellcolor{red!18}272.8\,{\tiny (8.9)} & \cellcolor{red!18}0.441\,{\tiny (0.064)} & \cellcolor{red!18}-149.8\,{\tiny (1.8)} & \cellcolor{red!18}1.189\,{\tiny (0.122)} & 4413\,{\tiny (405)} & \cellcolor{gray!10}1.081\,{\tiny (0.191)} \\
\midrule
GRAD & 2354\,{\tiny (348)} & \cellcolor{gray!10}1.180\,{\tiny (0.065)} & \textbf{3309}\,{\tiny (220)} & \cellcolor{gray!10}0.578\,{\tiny (0.050)} & \cellcolor{red!18}218.1\,{\tiny (18.0)} & \cellcolor{red!18}0.215\,{\tiny (0.021)} & -190.7\,{\tiny (16.2)} & \cellcolor{gray!10}0.451\,{\tiny (0.019)} & \cellcolor{red!18}3384\,{\tiny (928)} & \cellcolor{gray!10}0.626\,{\tiny (0.057)} & \cellcolor{red!18}3609\,{\tiny (918)} & \cellcolor{gray!10}1.940\,{\tiny (0.152)} & 2978\,{\tiny (616)} & \cellcolor{gray!10}0.898\,{\tiny (0.097)} & \cellcolor{red!18}250.2\,{\tiny (62.2)} & \cellcolor{gray!10}0.371\,{\tiny (0.087)} & -149.3\,{\tiny (1.4)} & \cellcolor{gray!10}0.504\,{\tiny (0.019)} & 4381\,{\tiny (278)} & \cellcolor{gray!10}0.905\,{\tiny (0.146)} \\
\midrule
ASAP & \underline{2807}\,{\tiny (411)} & \cellcolor{gray!10}1.111\,{\tiny (0.071)} & 2780\,{\tiny (557)} & \cellcolor{gray!10}\underline{0.441}\,{\tiny (0.035)} & \cellcolor{red!18}190.9\,{\tiny (28.3)} & \cellcolor{red!18}0.213\,{\tiny (0.014)} & \cellcolor{red!18}-223.2\,{\tiny (121.8)} & \cellcolor{gray!10}0.449\,{\tiny (0.040)} & \underline{3754}\,{\tiny (530)} & \cellcolor{gray!10}\underline{0.490}\,{\tiny (0.034)} & \cellcolor{red!18}3588\,{\tiny (790)} & \cellcolor{gray!10}\underline{1.383}\,{\tiny (0.168)} & \underline{3014}\,{\tiny (516)} & \cellcolor{gray!10}\underline{0.655}\,{\tiny (0.084)} & \cellcolor{red!18}275.7\,{\tiny (8.4)} & \cellcolor{gray!10}\underline{0.248}\,{\tiny (0.034)} & \textbf{-148.1}\,{\tiny (0.6)} & \cellcolor{gray!10}0.485\,{\tiny (0.041)} & 4466\,{\tiny (400)} & \cellcolor{gray!10}\underline{0.795}\,{\tiny (0.056)} \\
\midrule
LipsNet++ & 2297\,{\tiny (295)} & \cellcolor{gray!10}\underline{0.266}\,{\tiny (0.020)} & \cellcolor{red!18}1878\,{\tiny (642)} & \cellcolor{gray!10}0.543\,{\tiny (0.082)} & \cellcolor{red!18}134.5\,{\tiny (44.0)} & \cellcolor{gray!10}\textbf{0.144}\,{\tiny (0.021)} & \cellcolor{red!18}-671.4\,{\tiny (211.1)} & \cellcolor{red!18}0.553\,{\tiny (0.123)} & \cellcolor{red!18}2344\,{\tiny (556)} & \cellcolor{gray!10}0.993\,{\tiny (0.204)} & \textbf{4283}\,{\tiny (486)} & \cellcolor{gray!10}2.264\,{\tiny (0.190)} & \cellcolor{red!18}2570\,{\tiny (800)} & \cellcolor{gray!10}0.886\,{\tiny (0.159)} & \cellcolor{red!18}259.9\,{\tiny (46.6)} & \cellcolor{gray!10}0.358\,{\tiny (0.266)} & \underline{-148.4}\,{\tiny (0.7)} & \cellcolor{gray!10}0.694\,{\tiny (0.045)} & 4419\,{\tiny (197)} & \cellcolor{gray!10}1.093\,{\tiny (0.118)} \\
\midrule
\methodname{} & \textbf{2973}\,{\tiny (140)} & \cellcolor{gray!10}\textbf{0.127}\,{\tiny (0.006)} & \underline{3002}\,{\tiny (488)} & \cellcolor{gray!10}\textbf{0.222}\,{\tiny (0.033)} & \textbf{227.2}\,{\tiny (14.2)} & \cellcolor{gray!10}\underline{0.194}\,{\tiny (0.013)} & \underline{-183.9}\,{\tiny (12.7)} & \cellcolor{gray!10}\textbf{0.343}\,{\tiny (0.014)} & \textbf{4071}\,{\tiny (614)} & \cellcolor{gray!10}\textbf{0.407}\,{\tiny (0.032)} & \underline{4251}\,{\tiny (471)} & \cellcolor{gray!10}\textbf{0.557}\,{\tiny (0.049)} & \textbf{3046}\,{\tiny (514)} & \cellcolor{gray!10}\textbf{0.259}\,{\tiny (0.029)} & \textbf{278.2}\,{\tiny (2.4)} & \cellcolor{gray!10}\textbf{0.196}\,{\tiny (0.009)} & -149.0\,{\tiny (0.9)} & \cellcolor{gray!10}\textbf{0.291}\,{\tiny (0.024)} & \textbf{4714}\,{\tiny (314)} & \cellcolor{gray!10}\textbf{0.419}\,{\tiny (0.048)} \\
\bottomrule
\end{tabular}}
\vspace{-10pt}
\end{table*}

\begin{proposition}[The Temporal Penalty Bounds the Spatial Effect]
\label{prop:squared_euclidean_bound}
Assume the relevant second moments are finite. Then
\begin{equation}
    L_{\mathrm{S}}
    \le
    2L_{\mathrm{T}}.
    \label{eq:squared_euclidean_bound}
\end{equation}
\end{proposition}

\begin{proof}
Conditioning the temporal penalty on $s'$ gives
\begin{align*}
L_{\mathrm{T}}
&=\mathbb E_{s'\sim d_\pi^+}
  \mathbb E_{u\sim\rho_\pi(\cdot\mid s')}
  \left[\|\mu(u)-\mu(s')\|_2^2\right].
\end{align*}
For fixed $s'$, let $\rho=\rho_\pi(\cdot\mid s')$, let
$u,v\overset{\mathrm{i.i.d.}}{\sim}\rho$, and define
$\bar\mu(s')=\mathbb E_{u\sim\rho}[\mu(u)]$, where
$\bar\mu(s')$ is the mean action vector at current states conditioned on the
next state $s'$. The variance identity gives
\begin{equation*}
    \tfrac12
    \mathbb E_{u,v\sim\rho}\!
    \left[\|\mu(u)-\mu(v)\|_2^2\right]
    \!=\!
    \mathbb E_{u\sim\rho}\!
    \left[\|\mu(u)-\bar\mu(s')\|_2^2\right].
\end{equation*}
Applying the bias--variance decomposition gives
{
\begin{equation*}
\begin{aligned}
    &\mathbb E_{u\sim\rho}
    \left[\|\mu(u)-\mu(s')\|_2^2\right]\\
    &\quad=\tfrac12\mathbb E_{u,v\sim\rho}
    \left[\|\mu(u)-\mu(v)\|_2^2\right]
    +\|\bar\mu(s')-\mu(s')\|_2^2.
\end{aligned}
\end{equation*}}
Averaging over $s'\sim d_\pi^+$ then gives
\begin{multline*}
    L_{\mathrm{T}}-\frac{1}{2}L_{\mathrm{S}}
    \\
    =\mathbb E_{s'\sim d_\pi^+}
    \left[\|\bar\mu(s')-\mu(s')\|_2^2\right]
    \ge 0.
\end{multline*}
\end{proof}

Thus, reducing the temporal penalty tightens the upper bound on
$L_{\mathrm S}$, allowing temporal regularization to align actions between
states whose next-state distributions overlap. However, this alignment does not
by itself imply spatial smoothness.
Motivated by the smooth-transition structure of many continuous-state
environments~\cite{shen2020deep}, we hypothesize that nearby states often have
overlapping next-state distributions. Under this
assumption, action differences between nearby states contribute to
$L_{\mathrm S}$ across a range of shared next states. The temporal penalty can
therefore indirectly reduce these differences, promoting spatial smoothness
and lowering policy sensitivity to observation noise, as tested in
Section~\ref{subsec:observation_sensitivity}.

\subsection{Temporal Penalty with Linear Ramp-Up}
\label{subsec:linear_penalty_ramp_up}

The preceding analyses motivate using the temporal
penalty in \eqref{eq:squared_temporal_penalty_mean_actor} without explicit spatial
regularization. We next consider how to apply this penalty over the course of
training. Despite its narrower scope, strong smoothing requires a large
coefficient, which can still prevent the policy from selecting rewarding
actions when applied from the start of training. Because our goal is to obtain a
smooth final policy while preserving return rather than smooth behavior throughout
training, the coefficient can vary over the course of training.

These considerations motivate training-dependent schedules. Although
adaptive schedules and meta-optimization are possible, they add computation,
require a criterion for balancing return and smoothness during training, and may
not guarantee the desired final regularization strength. We begin this exploration
with a simple linear ramp-up that retains direct control
over the final regularization strength. By gradually increasing the penalty coefficient,
we allow the policy to discover high-return behavior before smoothing its
actions. For environment interaction count $k$,
training budget $K$, and final coefficient $\lambda$, the
\methodname{} actor objective is
\begin{equation}
\mathcal L_{\pi}
=
\mathcal L_{\mathrm{RL}}
+
\frac{k}{K}\lambda L_{\mathrm{T}},
\qquad 0 \le k \le K,
\label{eq:actor_loss_with_linear_penalty}
\end{equation}
where $\mathcal L_{\mathrm{RL}}$ is the base algorithm's actor loss.

\input{figures/observation_noise_hopper_table.tex}
\input{figures/observation_noise_walker_table.tex}

\begin{table*}[t]
\centering
\vspace*{4pt}
\caption{Observation-noise results.}
\label{tab:observation_noise_performance}
\scriptsize
\setlength{\tabcolsep}{1.5pt}
\renewcommand{\arraystretch}{1.1}
\setlength{\aboverulesep}{0.25ex}
\setlength{\belowrulesep}{0.25ex}
\begin{tabular*}{\textwidth}{@{\extracolsep{\fill}}>{\centering\arraybackslash}p{38pt} l *{6}{l >{\columncolor{gray!10}}l}}
\toprule
\multirow{2}{*}{Task} & \multirow{2}{*}{Method}
& \multicolumn{2}{c}{$\sigma=0$} & \multicolumn{2}{c}{$0.10$}
& \multicolumn{2}{c}{$0.20$} & \multicolumn{2}{c}{$0.30$}
& \multicolumn{2}{c}{$0.40$} & \multicolumn{2}{c}{$0.50$} \\
\cmidrule(lr){3-4} \cmidrule(lr){5-6} \cmidrule(lr){7-8}
\cmidrule(lr){9-10} \cmidrule(lr){11-12} \cmidrule(lr){13-14}
& & \multicolumn{1}{c}{$R \uparrow$} & \multicolumn{1}{c}{\cellcolor{gray!10}$sm \downarrow$}
& \multicolumn{1}{c}{$R \uparrow$} & \multicolumn{1}{c}{\cellcolor{gray!10}$sm \downarrow$}
& \multicolumn{1}{c}{$R \uparrow$} & \multicolumn{1}{c}{\cellcolor{gray!10}$sm \downarrow$}
& \multicolumn{1}{c}{$R \uparrow$} & \multicolumn{1}{c}{\cellcolor{gray!10}$sm \downarrow$}
& \multicolumn{1}{c}{$R \uparrow$} & \multicolumn{1}{c}{\cellcolor{gray!10}$sm \downarrow$}
& \multicolumn{1}{c}{$R \uparrow$} & \multicolumn{1}{c}{\cellcolor{gray!10}$sm \downarrow$} \\
\midrule
\ObservationNoiseHopperRows
\midrule[1pt]
\ObservationNoiseWalkerRows
\bottomrule
\end{tabular*}
\vspace{-10pt}
\end{table*}

\section{Experiments}

Our experiments test the claims made in Section I.
We evaluate \methodname{}'s ability to reduce
oscillation while preserving return under clean and noisy observations,
together with its practical cost (Section~\ref{subsec:main_results}).
We compare temporal and spatial regularization under strong smoothing
(Section~\ref{subsec:temporal_spatial_regularization}), validate the predicted
spatial effect of the temporal penalty
(Section~\ref{subsec:transition_conditioned_experiment}), and
test its extension to spatial smoothness
(Section~\ref{subsec:observation_sensitivity}). We examine the benefits of
ramp-up for return and both forms of smoothness
(Section~\ref{subsec:linear_ramp_up_ablation}), followed by sim-to-real
validation on a physical robot (Section~\ref{subsec:sim2real_evaluation}).

\subsection{Experimental Setup}
We evaluated \methodname{} on five continuous-control tasks from
Gymnasium~\cite{NEURIPS2025_d7ff1795} and MuJoCo~\cite{todorov2012mujoco},
including Ant-v5, Hopper-v5, LunarLander-v3 with continuous actions,
Pendulum-v1, and Walker2d-v5. We applied \methodname{} to
SAC~\cite{haarnoja2018soft} and PPO~\cite{schulman2017proximal} and compared it
with the unregularized algorithms, CAPS~\cite{mysore2021regularizing},
L2C2~\cite{kobayashi2022l2c2}, Grad-CAPS (GRAD)~\cite{lee2024gradient},
ASAP~\cite{kwak2026enhancing}, and LipsNet++~\cite{song2025lipsnetpp}.
For CAPS, L2C2, Grad-CAPS, and ASAP, we used the configurations collected by
ASAP. For LipsNet++ and CATS, we used the same tuning
budget and selected the smoothest setting whose mean return was at least Base's
95\% confidence lower bound, or the highest-return setting otherwise. For LipsNet++,
the search covered all eight combinations of its DMControl hyperparameter settings.
We used the tuned SAC and PPO hyperparameters from RL Baselines3
Zoo to evaluate each regularizer with a competitive
base algorithm. Unless otherwise noted, we report means and 95\% confidence
intervals across ten training runs, each evaluated on ten held-out seeds.
We report cumulative return for task performance
and the smoothness score for action smoothness, and red-shaded entries indicate
results worse than Base.

\subsection{Main Results}
\label{subsec:main_results}
We first trained and evaluated \methodname{}, Base, and existing regularizers under
clean observations to assess their return--smoothness trade-offs.
Table~\ref{tab:main_results} shows that \methodname{} reduced the
smoothness score by $63.2\%$ on average and by up to $91.3\%$ on PPO Ant relative
to Base.
Although prior work has observed that smoothness
regularization can improve return, CATS substantially reduced action oscillation
while preserving return. This advantage was particularly pronounced on
higher-dimensional tasks, where broad spatial regularization may be more
restrictive than regularization along sampled transition directions. On Ant,
Hopper, and Walker, CATS was the only evaluated regularizer that did not reduce
mean return relative to Base. Notably, on
SAC Ant, the other evaluated regularizers either reduced return or provided
little improvement in smoothness, whereas CATS alone substantially improved
smoothness without reducing return.
The importance of temporal regularization and linear ramp-up in achieving this
result is examined in Sections~\ref{subsec:temporal_spatial_regularization_exp}
and~\ref{subsec:linear_ramp_up_ablation}, respectively.

We next evaluated the policies from Table~\ref{tab:main_results} under
observation noise on Hopper and Walker to test whether \methodname{} preserves
these gains without spatial regularization.
Table~\ref{tab:observation_noise_performance} reports this evaluation. We added independent
feature-wise uniform noise within $\pm\sigma$ times Base's clean-observation
standard deviation to unnormalized observations so that the perturbation
reflects each feature's scale.
Competing methods often sacrificed return for smoothness under clean
observations, or fell below Base in return as noise increased,
failing to preserve task performance under the conditions they are intended to
address. In contrast, \methodname{} remained smoother than Base while maintaining
higher return in every setting, supporting its practical utility
across diverse noise scales. Section~\ref{subsec:observation_sensitivity}
directly examines CATS's low policy sensitivity underlying this result.

To compare practical cost, we optimized each method's implementation within the
public ASAP codebase and measured end-to-end wall-clock
training time on PPO Walker over 1M
environment steps using an NVIDIA RTX 3090 GPU (24 GB).
Table~\ref{tab:ppo_walker_training_overhead} reports the mean
overhead relative to Base across four seeds, together with the
number of method-specific hyperparameters. At $19.8\%$, \methodname{} had the
lowest measured training overhead among the evaluated regularizers and required
only one such hyperparameter, supporting straightforward integration into
existing RL algorithms.

\begin{table}[t]
\centering
\caption{Training overhead.}
\label{tab:ppo_walker_training_overhead}
\scriptsize
\setlength{\tabcolsep}{2.6pt}
\renewcommand{\arraystretch}{1.1}
\setlength{\aboverulesep}{0.25ex}
\setlength{\belowrulesep}{0.25ex}
\begin{tabular*}{\columnwidth}{@{\extracolsep{\fill}}lcccccc@{}}
\toprule
& LipsNet++ & L2C2 & ASAP & GRAD & CAPS & \methodname{} \\
\midrule
Overhead (\%) & 83.5 & 77.1 & 40.5 & 30.2 & 28.6 & \textbf{19.8} \\
Hyperparameters & 3 & 4 & 3 & 1 & 3 & \textbf{1} \\
\bottomrule
\end{tabular*}
\vspace{-10pt}
\end{table}

\begin{figure}[t]
    \centering
    \includegraphics[width=\columnwidth]{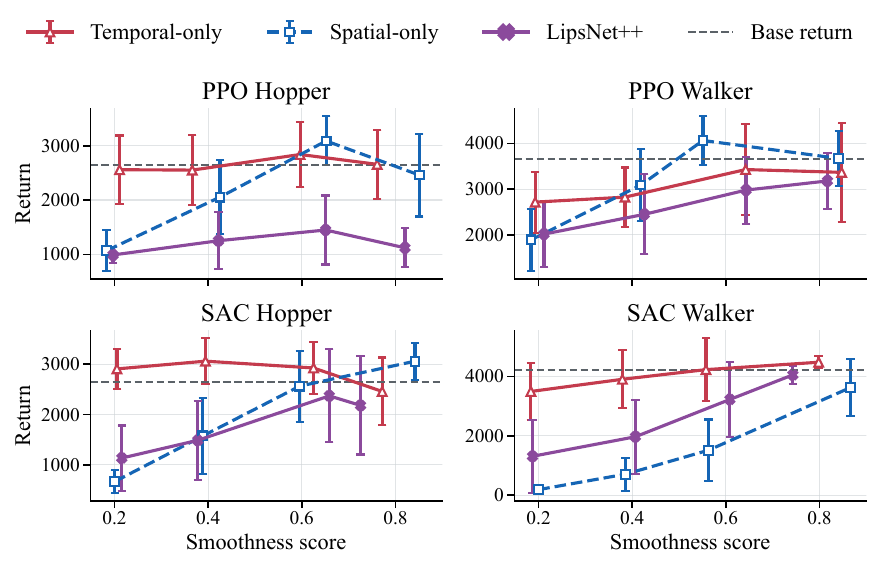}
    \vspace{-10pt}
    \caption{Return at matched smoothness levels on Hopper and Walker. For
    each target score in $\{0.2,0.4,0.6,0.8\}$, we compared configurations whose
    smoothness scores (lower is better) are within 10\% of the target. Error
    bars indicate 95\% confidence intervals.}
    \label{fig:regularization_return_smoothness_profiles}
\vspace{-10pt}
\end{figure}

\subsection{Temporal and Spatial Regularization}
\label{subsec:temporal_spatial_regularization_exp}
\label{subsec:temporal_spatial_regularization}
To test the basis for selecting temporal regularization in CATS, we compared
temporal-only CAPS, spatial-only CAPS, and LipsNet++ using the same training and
evaluation settings as Table~\ref{tab:main_results}. We obtained four smoothness
levels by increasing each method's penalty coefficient; for LipsNet++, we varied
only the Jacobian-penalty coefficient $\lambda_k$. In
Fig.~\ref{fig:regularization_return_smoothness_profiles}, the methods achieved
similar returns under moderate smoothing, but temporal-only CAPS preserved
substantially more return as smoothing became stronger. This result supports
our choice of temporal regularization, as constraining action differences along
sampled transition directions enables stronger smoothing with less performance
degradation.

\begin{table}[t]
\centering
\caption{Empirical $\widehat L_{\mathrm S}$ for neighborhood sizes
$k\in\{4,32,128\}$.}
\label{tab:transition_conditioned_validation}
\scriptsize
\setlength{\tabcolsep}{2.5pt}
\renewcommand{\arraystretch}{1.1}
\setlength{\aboverulesep}{0.25ex}
\setlength{\belowrulesep}{0.25ex}
\begin{tabular}{l l ccc}
\toprule
Task & Method & $k=4$ & $k=32$ & $k=128$ \\
\midrule
\multirow{4}{*}{Hopper} & PPO & 0.384\,{\tiny (0.081)} & 0.403\,{\tiny (0.084)} & 0.421\,{\tiny (0.087)} \\
& \methodname{} & 0.006\,{\tiny (0.001)} & 0.009\,{\tiny (0.001)} & 0.012\,{\tiny (0.002)} \\
\cmidrule{2-5}
& SAC & 0.172\,{\tiny (0.022)} & 0.194\,{\tiny (0.023)} & 0.214\,{\tiny (0.024)} \\
& \methodname{} & 0.007\,{\tiny (0.002)} & 0.010\,{\tiny (0.003)} & 0.013\,{\tiny (0.003)} \\
\midrule[1pt]
\multirow{4}{*}{Walker} & PPO & 0.535\,{\tiny (0.069)} & 0.608\,{\tiny (0.077)} & 0.665\,{\tiny (0.084)} \\
& \methodname{} & 0.038\,{\tiny (0.013)} & 0.054\,{\tiny (0.018)} & 0.070\,{\tiny (0.022)} \\
\cmidrule{2-5}
& SAC & 0.253\,{\tiny (0.058)} & 0.303\,{\tiny (0.068)} & 0.347\,{\tiny (0.077)} \\
& \methodname{} & 0.028\,{\tiny (0.008)} & 0.039\,{\tiny (0.011)} & 0.049\,{\tiny (0.014)} \\
\bottomrule
\end{tabular}
\vspace{-10pt}
\end{table}

\subsection{Empirical Validation of the Spatial Effect}
\label{subsec:transition_conditioned_pairwise_smoothing}
\label{subsec:transition_conditioned_experiment}
To empirically examine whether temporal regularization
reduces the action differences bounded by
Proposition~\ref{prop:squared_euclidean_bound}, we collected 10,000 stochastic
episodes for each policy, yielding between 5.3 and 9.9 million transitions
across settings, and sampled 20,000 query transitions $(s,a,s')$. For each query, we retrieved the
$k\in\{4,32,128\}$ transitions $(\tilde s_i,\tilde a_i,\tilde s'_i)$ from other
episodes whose next states are nearest to $s'$ in normalized next-state space,
thereby approximating conditioning current states on a shared next state. We then
estimated $\widehat L_{\mathrm S}$ by averaging
$\|\mu(\tilde s_i)-\mu(\tilde s_j)\|_2^2$ over cross-episode pairs
within each neighborhood.

Table~\ref{tab:transition_conditioned_validation} shows that CATS achieved lower
$\widehat L_{\mathrm S}$ than Base for every neighborhood size and
algorithm--environment setting. The temporal penalty alone indeed reduced action differences
between current states paired by next-state proximity, without directly
constraining their actions against one another.

\subsection{Policy Sensitivity to Observation Noise}
\label{subsec:observation_sensitivity}
In
Section~\ref{subsec:dynamics_aware_temporal_penalty}, we hypothesized that nearby
states share possible next states, allowing temporal regularization to indirectly
reduce their action differences and thereby promote spatial smoothness.
To test this hypothesis, we measured half the mean squared difference between
clean and perturbed policy outputs over ten noise draws at each visited
observation, using the protocol in
Table~\ref{tab:observation_noise_performance}.
Figure~\ref{fig:policy_output_sensitivity} shows that CATS
achieved policy sensitivity comparable to several explicit
spatial regularization methods except on SAC LunarLander, despite using no
spatial penalty.
Notably, on Hopper, CATS and LipsNet++ achieved
comparable policy sensitivity, but Table~\ref{tab:main_results} shows that only
CATS preserved task performance.
These results show that spatial smoothness is not exclusive to spatial
regularization and suggest that performance degradation stems from
directly constraining action changes across multiple
state-space directions rather than from spatial smoothness itself.

\begin{figure*}[t]
    \centering
    \vspace*{3pt}
    \includegraphics[width=0.36\textwidth,trim=0bp 3.7bp 0bp 2.9bp,clip]{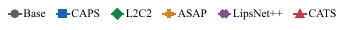}\par
    \nointerlineskip\vspace{1pt}
    \includegraphics[width=\textwidth]{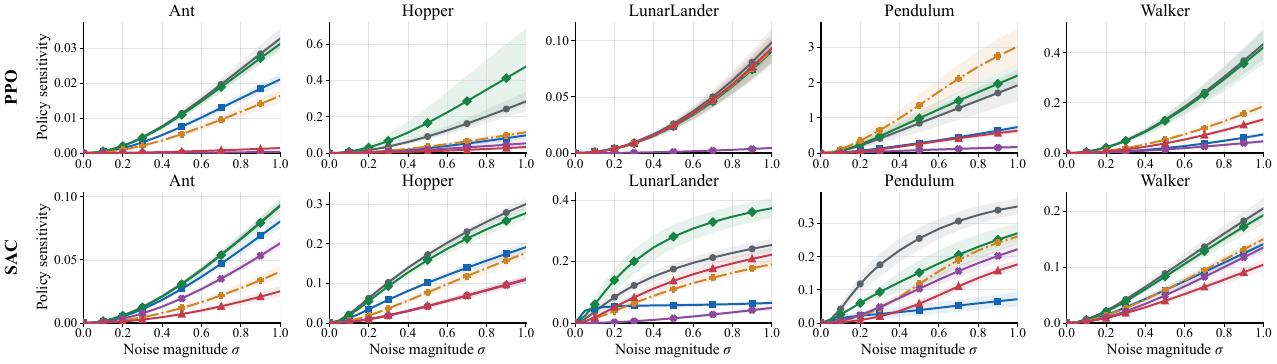}
    \vspace{-10pt}
    \caption{Policy sensitivity under observation noise for PPO (top)
    and SAC (bottom). Shaded regions indicate 95\% confidence intervals.}
    \label{fig:policy_output_sensitivity}
\vspace{-10pt}
\end{figure*}

\begin{figure}[t]
    \centering
    \includegraphics[width=\columnwidth,trim=0bp 5bp 0bp 0bp,clip]{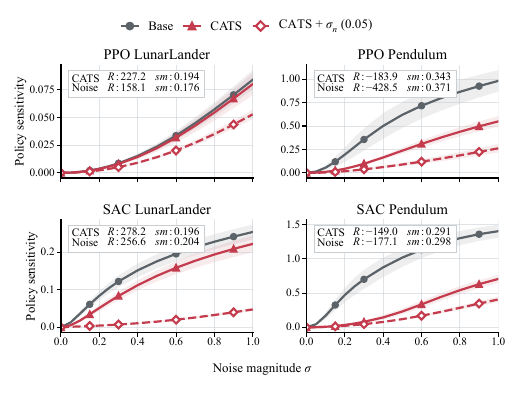}
    \vspace{-10pt}
    \caption{Next-state-noise extension on LunarLander and Pendulum. $R$
    denotes return and $sm$ the smoothness score. Shaded regions indicate
    95\% confidence intervals.}
    \label{fig:next_state_noise_extension}
\vspace{-10pt}
\end{figure}

This induced spatial smoothness, however, depends on overlap between next-state
distributions and may weaken when the overlap is limited.
To illustrate how this spatial effect can be extended,
we added Gaussian noise to the next state used by the temporal penalty on
LunarLander and Pendulum.
Figure~\ref{fig:next_state_noise_extension}
shows that broadening the next-state distribution indeed
led to lower policy sensitivity.
Unlike existing methods that combine temporal and
spatial penalties, this extension uses the spatial effect of temporal
regularization to adjust the range of spatial smoothness without additional
actor forward passes or a relative weighting between penalties. However,
consistent with our scope analysis, broadening this range also reduced task
performance, and mitigating this trade-off remains future work.

\subsection{Linear Ramp-Up Ablation}
\label{subsec:linear_ramp_up_ablation}

\begin{table}[t]
\centering
\vspace*{4pt}
\caption{Effect of ramp-up on return and smoothness.}
\label{tab:penalty_ramp_up_controls}
\scriptsize
\setlength{\tabcolsep}{2pt}
\renewcommand{\arraystretch}{1.1}
\setlength{\aboverulesep}{0.25ex}
\setlength{\belowrulesep}{0.25ex}
\begin{tabular*}{\columnwidth}{@{\extracolsep{\fill}}l l >{\columncolor{gray!10}}l l >{\columncolor{gray!10}}l l >{\columncolor{gray!10}}l}
\toprule
\multirow{2}{*}{Method} & \multicolumn{2}{c}{Ant} & \multicolumn{2}{c}{Hopper} & \multicolumn{2}{c}{Walker} \\
\cmidrule(lr){2-3} \cmidrule(lr){4-5} \cmidrule(lr){6-7}
& \multicolumn{1}{c}{$R \uparrow$} & \multicolumn{1}{c}{\cellcolor{gray!10}$sm \downarrow$} & \multicolumn{1}{c}{$R \uparrow$} & \multicolumn{1}{c}{\cellcolor{gray!10}$sm \downarrow$} & \multicolumn{1}{c}{$R \uparrow$} & \multicolumn{1}{c}{\cellcolor{gray!10}$sm \downarrow$} \\
\midrule
PPO & 1411\,{\tiny (506)} & \cellcolor{gray!10}1.464\,{\tiny (0.193)} & 2658\,{\tiny (685)} & \cellcolor{gray!10}1.162\,{\tiny (0.265)} & 3510\,{\tiny (582)} & \cellcolor{gray!10}1.682\,{\tiny (0.237)} \\
\midrule
Fixed-middle & \underline{3024}\,{\tiny (304)} & \cellcolor{gray!10}0.163\,{\tiny (0.012)} & 2784\,{\tiny (639)} & \cellcolor{gray!10}0.266\,{\tiny (0.039)} & \underline{3777}\,{\tiny (339)} & \cellcolor{gray!10}0.459\,{\tiny (0.030)} \\
Fixed-final & 2973\,{\tiny (142)} & \cellcolor{gray!10}\textbf{0.124}\,{\tiny (0.005)} & 2664\,{\tiny (687)} & \cellcolor{gray!10}\textbf{0.211}\,{\tiny (0.016)} & \cellcolor{red!18}3188\,{\tiny (1058)} & \cellcolor{gray!10}\textbf{0.342}\,{\tiny (0.038)} \\
\methodname{} & 2973\,{\tiny (140)} & \cellcolor{gray!10}\underline{0.127}\,{\tiny (0.006)} & \textbf{3002}\,{\tiny (488)} & \cellcolor{gray!10}\underline{0.222}\,{\tiny (0.033)} & \textbf{4071}\,{\tiny (614)} & \cellcolor{gray!10}\underline{0.407}\,{\tiny (0.032)} \\
\midrule[1pt]
SAC & 3656\,{\tiny (675)} & \cellcolor{gray!10}2.286\,{\tiny (0.213)} & 2656\,{\tiny (702)} & \cellcolor{gray!10}1.507\,{\tiny (0.273)} & 4189\,{\tiny (676)} & \cellcolor{gray!10}1.268\,{\tiny (0.126)} \\
\midrule
Fixed-middle & \cellcolor{red!18}3129\,{\tiny (556)} & \cellcolor{gray!10}\underline{0.532}\,{\tiny (0.109)} & 2807\,{\tiny (607)} & \cellcolor{gray!10}0.314\,{\tiny (0.015)} & \underline{4465}\,{\tiny (265)} & \cellcolor{gray!10}0.433\,{\tiny (0.047)} \\
Fixed-final & \cellcolor{red!18}2572\,{\tiny (479)} & \cellcolor{gray!10}\textbf{0.290}\,{\tiny (0.076)} & \underline{2893}\,{\tiny (439)} & \cellcolor{gray!10}\underline{0.262}\,{\tiny (0.031)} & 4362\,{\tiny (649)} & \cellcolor{gray!10}\textbf{0.398}\,{\tiny (0.065)} \\
\methodname{} & \textbf{4251}\,{\tiny (471)} & \cellcolor{gray!10}0.557\,{\tiny (0.049)} & \textbf{3046}\,{\tiny (514)} & \cellcolor{gray!10}\textbf{0.259}\,{\tiny (0.029)} & \textbf{4714}\,{\tiny (314)} & \cellcolor{gray!10}\underline{0.419}\,{\tiny (0.048)} \\
\bottomrule
\end{tabular*}
\vspace{-10pt}
\end{table}
\input{figures/ramp_sensitivity_table.tex}

Having established temporal regularization as the basis of CATS, we test whether
linear ramp-up improves return preservation and both temporal and spatial
smoothness. Table~\ref{tab:penalty_ramp_up_controls} compares CATS with two
fixed-coefficient controls on Ant, Hopper, and Walker under PPO and SAC.
\emph{Fixed-middle} used $\lambda/2$
throughout training, whereas
\emph{Fixed-final} used $\lambda$, matching CATS's final coefficient.
CATS achieved higher return than both fixed
controls in five of the six settings, while generally matching or improving
on the smoothness of Fixed-middle. On SAC Ant, CATS alone preserved Base
return, whereas both fixed controls fell below it. This suggests that relaxing
early regularization is important for allowing the policy to discover actions
that yield higher return.

Despite using a lower coefficient throughout most of training,
CATS achieved smoothness similar to Fixed-final in four of the six settings.
One possible explanation is that learning rewarding behavior before
smoothing it reduces conflict between reward optimization and smoothing,
compared with searching for rewarding actions under strong regularization
from the start.
Compared with Fixed-middle, which uses the same average penalty coefficient,
CATS achieved lower smoothness scores in five settings. This difference
may reflect the greater influence of later regularization on the final policy,
whereas subsequent policy updates may diminish the effects of early regularization.
Table~\ref{tab:ramp_sensitivity} further suggests that ramp-up improves not only
temporal smoothness but also the induced spatial effect, yielding lower policy
sensitivity.

\begin{table}[t]
\centering
\vspace*{4pt}
\caption{Sim-to-real results. Real-world success rates are mean $\pm$ standard
error, and all parenthetical values are 95\% confidence intervals.}
\label{tab:sim2real_results}
\label{tab:sim2real_lift_results}
\scriptsize
\setlength{\tabcolsep}{1.3pt}
\renewcommand{\arraystretch}{1.1}
\setlength{\aboverulesep}{0.25ex}
\setlength{\belowrulesep}{0.25ex}
\begin{tabular*}{\columnwidth}{@{\extracolsep{\fill}}l l ccc cc@{}}
\toprule
& Method & \multicolumn{3}{c}{Reach} & \multicolumn{2}{c}{Place Cube} \\
\cmidrule(lr){3-5}\cmidrule(lr){6-7}
& & S@5cm $\uparrow$ & S@10cm $\uparrow$ & $sm \downarrow$
& Success $\uparrow$ & $sm \downarrow$ \\
\midrule
\multirow{3}{*}{\textit{Sim}} & PPO & 88.5\,{\tiny (16.3)} & \textbf{100.0}\,{\tiny (0.0)} & 4.345\,{\tiny (2.897)} & \textbf{99.3}\,{\tiny (1.5)} & 1.700\,{\tiny (0.599)} \\
& CAPS & \underline{99.9}\,{\tiny (0.4)} & \textbf{100.0}\,{\tiny (0.0)} & \textbf{0.106}\,{\tiny (0.034)} & 98.9\,{\tiny (1.1)} & \underline{0.260}\,{\tiny (0.066)} \\
& \methodname{} & \textbf{100.0}\,{\tiny (0.0)} & \textbf{100.0}\,{\tiny (0.0)} & \textbf{0.106}\,{\tiny (0.025)} & \underline{99.1}\,{\tiny (2.3)} & \textbf{0.221}\,{\tiny (0.050)} \\
\midrule[1pt]
\multirow{3}{*}{\textit{Real}} & PPO & \underline{2.8}\,$\pm$\,{\tiny 2.7} & 8.3\,$\pm$\,{\tiny 4.6} & 6.104\,{\tiny (0.604)} & 40.0\,$\pm$\,{\tiny 15.5} & 0.977\,{\tiny (0.097)} \\
& CAPS & 0.0\,$\pm$\,{\tiny 0.0} & \underline{50.0}\,$\pm$\,{\tiny 8.3} & \textbf{0.081}\,{\tiny (0.018)} & \textbf{60.0}\,$\pm$\,{\tiny 15.5} & \underline{0.094}\,{\tiny (0.011)} \\
& \methodname{} & \textbf{16.7}\,$\pm$\,{\tiny 6.2} & \textbf{72.2}\,$\pm$\,{\tiny 7.5} & \underline{0.105}\,{\tiny (0.014)} & \underline{50.0}\,$\pm$\,{\tiny 15.8} & \textbf{0.088}\,{\tiny (0.009)} \\
\bottomrule
\end{tabular*}
\vspace{-10pt}
\end{table}

\subsection{Sim-to-Real Evaluation}
\label{subsec:sim2real_evaluation}
To assess whether the benefits of CATS extend to
robotic tasks and real-world deployment, we trained PPO policies for Reach and
Place Cube in Isaac Lab.
Reach requires moving the end effector to a target position in free space,
whereas Place Cube starts with the cube grasped and
requires placing it at a target position on the table.
Reach success required staying within 5 or 10 cm of the
target for one second, whereas Place Cube success required placing the cube
within 5 cm of the target.
Policies received target-conditioned images, target coordinates, and
proprioceptive observations following Azimi et al.~\cite{azimi2026benchmarking}.
We compared PPO, CAPS, and CATS to test whether the temporal-only design of
CATS transfers without the explicit spatial penalty used by CAPS. We used five
training seeds for Reach and three for Place Cube, with five evaluation seeds
per policy in both tasks.
We omitted reward penalties to clarify how each regularizer changes
PPO's task-driven behavior.
For both tasks, we selected each method's simulation policy with the highest
5-cm success rate and deployed it on a Franka Research 3 without domain
randomization. On the real robot, each
method was evaluated on 36 Reach targets and 10 Place Cube targets. For safe PPO deployment,
we applied position scaling and impedance control equally to all methods.

Table~\ref{tab:sim2real_results} shows that, in simulation, CAPS and CATS reduced
action oscillation relative to PPO on both tasks while preserving success rates.
Figure~\ref{fig:sim2real_action_trace} shows
a representative Reach example. On real Reach, we observed that PPO produced larger
oscillatory motions than in simulation, preventing the end effector from staying
near the target and degrading success and smoothness. Place Cube was less
affected because the arm only had to approach the target and release the cube. In
contrast, CAPS and CATS retained smoother action patterns after transfer,
limiting abrupt motions. CATS also preserved task performance relative to PPO,
achieving higher success at both Reach thresholds and on Place Cube while
recording the lowest Place Cube smoothness score. Although CAPS succeeded on
one more Place Cube trial than CATS, these results show that CATS remains effective in
real-world deployment despite its simple temporal-only design.

\begin{figure}[t]
    \centering
    \includegraphics[width=\columnwidth,clip]{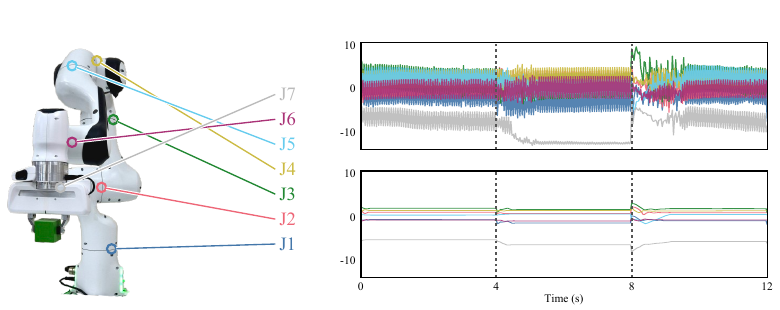}
    \vspace{-10pt}
    \caption{Franka robot with color-coded joints (left) and corresponding raw
    policy outputs during Reach (right).
    After each target change (dashed lines), CATS (bottom) responds with smaller action
    changes and quickly settles to stable outputs, whereas PPO (top) exhibits larger
    action changes followed by persistent oscillations.}
    \label{fig:sim2real_action_trace}
\vspace{-10pt}
\end{figure}

\section{Conclusion}
Our findings show that temporal regularization can
provide both temporal and spatial smoothness while better
preserving return under strong smoothing. More broadly, these findings
motivate designing smoothness regularization around transition directions
relevant to control, avoiding unnecessary constraints on action changes needed
for task performance. Linear ramp-up complements this direction by allowing the
policy to learn rewarding behavior before smoothing its actions, rather than
seeking rewarding behavior under smoothness constraints. \methodname{}
demonstrates the practical value of these findings through smooth
control in simulation and on a physical robot with little additional complexity.

\bibliographystyle{IEEEtran}
\bibliography{reference}

\end{document}

%% file: figures/observation_noise_hopper_table.tex
\newcommand{\ObservationNoiseHopperRows}{%
\multirow{12}{*}[-1.5ex]{Hopper} & PPO & 2658\,{\tiny (685)} & 1.162\,{\tiny (0.265)} & \underline{2681}\,{\tiny (599)} & 1.491\,{\tiny (0.252)} & \underline{2299}\,{\tiny (428)} & 1.730\,{\tiny (0.239)} & 1852\,{\tiny (321)} & 1.832\,{\tiny (0.233)} & 1522\,{\tiny (306)} & 1.878\,{\tiny (0.245)} & 1343\,{\tiny (251)} & 1.936\,{\tiny (0.224)} \\
\cmidrule{2-14}
& CAPS & \cellcolor{red!18}2440\,{\tiny (608)} & 0.471\,{\tiny (0.049)} & \cellcolor{red!18}2337\,{\tiny (538)} & 0.552\,{\tiny (0.056)} & \cellcolor{red!18}2205\,{\tiny (383)} & 0.693\,{\tiny (0.068)} & \underline{2040}\,{\tiny (327)} & 0.835\,{\tiny (0.098)} & \underline{1729}\,{\tiny (416)} & 0.930\,{\tiny (0.142)} & \underline{1500}\,{\tiny (386)} & 1.005\,{\tiny (0.157)} \\
\cmidrule{2-14}
& L2C2 & 2704\,{\tiny (545)} & \cellcolor{red!18}1.502\,{\tiny (0.435)} & \cellcolor{red!18}2425\,{\tiny (414)} & \cellcolor{red!18}1.632\,{\tiny (0.400)} & \cellcolor{red!18}2236\,{\tiny (292)} & \cellcolor{red!18}1.808\,{\tiny (0.270)} & 1987\,{\tiny (344)} & \cellcolor{red!18}1.942\,{\tiny (0.274)} & 1621\,{\tiny (304)} & \cellcolor{red!18}1.957\,{\tiny (0.228)} & 1404\,{\tiny (246)} & \cellcolor{red!18}1.981\,{\tiny (0.215)} \\
\cmidrule{2-14}
& ASAP & \underline{2780}\,{\tiny (557)} & \underline{0.441}\,{\tiny (0.035)} & \cellcolor{red!18}2597\,{\tiny (481)} & \underline{0.542}\,{\tiny (0.059)} & \cellcolor{red!18}2172\,{\tiny (251)} & \underline{0.675}\,{\tiny (0.061)} & 1959\,{\tiny (171)} & \underline{0.807}\,{\tiny (0.066)} & 1673\,{\tiny (198)} & \underline{0.894}\,{\tiny (0.084)} & 1343\,{\tiny (112)} & \underline{0.929}\,{\tiny (0.061)} \\
\cmidrule{2-14}
& LipsNet++ & \cellcolor{red!18}1878\,{\tiny (642)} & 0.543\,{\tiny (0.082)} & \cellcolor{red!18}1881\,{\tiny (609)} & 0.620\,{\tiny (0.071)} & \cellcolor{red!18}1670\,{\tiny (602)} & 0.708\,{\tiny (0.099)} & \cellcolor{red!18}1487\,{\tiny (528)} & 0.813\,{\tiny (0.117)} & \cellcolor{red!18}1347\,{\tiny (465)} & 0.916\,{\tiny (0.124)} & \cellcolor{red!18}1228\,{\tiny (395)} & 0.999\,{\tiny (0.131)} \\
\cmidrule{2-14}
& \methodname{} & \textbf{3002}\,{\tiny (488)} & \textbf{0.222}\,{\tiny (0.033)} & \textbf{2720}\,{\tiny (515)} & \textbf{0.287}\,{\tiny (0.016)} & \textbf{2428}\,{\tiny (499)} & \textbf{0.393}\,{\tiny (0.032)} & \textbf{2249}\,{\tiny (500)} & \textbf{0.493}\,{\tiny (0.053)} & \textbf{2006}\,{\tiny (446)} & \textbf{0.572}\,{\tiny (0.059)} & \textbf{1743}\,{\tiny (403)} & \textbf{0.627}\,{\tiny (0.070)} \\
\cmidrule[1pt]{2-14}
& SAC & 2656\,{\tiny (702)} & 1.507\,{\tiny (0.273)} & 2624\,{\tiny (484)} & 2.199\,{\tiny (0.237)} & 2176\,{\tiny (548)} & 2.479\,{\tiny (0.319)} & 1777\,{\tiny (353)} & 2.585\,{\tiny (0.233)} & 1372\,{\tiny (250)} & 2.549\,{\tiny (0.205)} & 1114\,{\tiny (161)} & 2.501\,{\tiny (0.157)} \\
\cmidrule{2-14}
& CAPS & 2856\,{\tiny (393)} & 1.235\,{\tiny (0.087)} & 2689\,{\tiny (398)} & 1.664\,{\tiny (0.142)} & \cellcolor{red!18}2063\,{\tiny (486)} & 1.838\,{\tiny (0.275)} & \cellcolor{red!18}1653\,{\tiny (439)} & 1.905\,{\tiny (0.311)} & \cellcolor{red!18}1345\,{\tiny (356)} & 1.933\,{\tiny (0.294)} & \cellcolor{red!18}1063\,{\tiny (270)} & 1.882\,{\tiny (0.276)} \\
\cmidrule{2-14}
& L2C2 & 2965\,{\tiny (614)} & 1.249\,{\tiny (0.105)} & 2659\,{\tiny (596)} & 1.666\,{\tiny (0.140)} & \cellcolor{red!18}2083\,{\tiny (413)} & 1.988\,{\tiny (0.202)} & \cellcolor{red!18}1679\,{\tiny (425)} & 2.101\,{\tiny (0.319)} & \cellcolor{red!18}1319\,{\tiny (278)} & 2.125\,{\tiny (0.296)} & 1233\,{\tiny (307)} & 2.239\,{\tiny (0.339)} \\
\cmidrule{2-14}
& ASAP & \underline{3014}\,{\tiny (516)} & \underline{0.655}\,{\tiny (0.084)} & \textbf{2919}\,{\tiny (490)} & 1.047\,{\tiny (0.096)} & \underline{2427}\,{\tiny (460)} & 1.401\,{\tiny (0.155)} & \underline{1923}\,{\tiny (355)} & 1.592\,{\tiny (0.172)} & \underline{1680}\,{\tiny (324)} & 1.764\,{\tiny (0.180)} & \underline{1427}\,{\tiny (241)} & 1.843\,{\tiny (0.162)} \\
\cmidrule{2-14}
& LipsNet++ & \cellcolor{red!18}2570\,{\tiny (800)} & 0.886\,{\tiny (0.159)} & \cellcolor{red!18}2534\,{\tiny (747)} & \underline{1.035}\,{\tiny (0.182)} & \cellcolor{red!18}2161\,{\tiny (599)} & \underline{1.177}\,{\tiny (0.188)} & 1804\,{\tiny (503)} & \underline{1.277}\,{\tiny (0.208)} & 1579\,{\tiny (421)} & \underline{1.380}\,{\tiny (0.230)} & 1384\,{\tiny (433)} & \underline{1.457}\,{\tiny (0.231)} \\
\cmidrule{2-14}
& \methodname{} & \textbf{3046}\,{\tiny (514)} & \textbf{0.259}\,{\tiny (0.029)} & \underline{2880}\,{\tiny (392)} & \textbf{0.585}\,{\tiny (0.036)} & \textbf{2811}\,{\tiny (276)} & \textbf{0.934}\,{\tiny (0.068)} & \textbf{2488}\,{\tiny (279)} & \textbf{1.164}\,{\tiny (0.078)} & \textbf{2150}\,{\tiny (312)} & \textbf{1.325}\,{\tiny (0.095)} & \textbf{1846}\,{\tiny (262)} & \textbf{1.429}\,{\tiny (0.081)} \\
}

%% file: figures/observation_noise_walker_table.tex
\newcommand{\ObservationNoiseWalkerRows}{%
\multirow{12}{*}[-1.5ex]{Walker} & PPO & 3510\,{\tiny (582)} & 1.682\,{\tiny (0.237)} & \underline{3706}\,{\tiny (372)} & 2.002\,{\tiny (0.165)} & 3301\,{\tiny (504)} & 2.227\,{\tiny (0.129)} & 2851\,{\tiny (344)} & 2.379\,{\tiny (0.182)} & 1940\,{\tiny (475)} & 2.293\,{\tiny (0.211)} & 1652\,{\tiny (342)} & 2.307\,{\tiny (0.150)} \\
\cmidrule{2-14}
& CAPS & 3626\,{\tiny (666)} & 0.505\,{\tiny (0.056)} & \cellcolor{red!18}3599\,{\tiny (697)} & \underline{0.564}\,{\tiny (0.062)} & \textbf{3831}\,{\tiny (351)} & \underline{0.706}\,{\tiny (0.044)} & \underline{3327}\,{\tiny (382)} & \textbf{0.809}\,{\tiny (0.053)} & 2553\,{\tiny (490)} & \textbf{0.853}\,{\tiny (0.081)} & \underline{2398}\,{\tiny (570)} & \textbf{0.961}\,{\tiny (0.110)} \\
\cmidrule{2-14}
& L2C2 & \cellcolor{red!18}3342\,{\tiny (547)} & \cellcolor{red!18}1.744\,{\tiny (0.157)} & \cellcolor{red!18}3353\,{\tiny (446)} & 1.969\,{\tiny (0.149)} & \cellcolor{red!18}2841\,{\tiny (480)} & 2.095\,{\tiny (0.192)} & \cellcolor{red!18}2429\,{\tiny (363)} & 2.226\,{\tiny (0.155)} & 1954\,{\tiny (378)} & 2.252\,{\tiny (0.191)} & \cellcolor{red!18}1554\,{\tiny (247)} & 2.244\,{\tiny (0.175)} \\
\cmidrule{2-14}
& ASAP & \underline{3754}\,{\tiny (530)} & \underline{0.490}\,{\tiny (0.034)} & \cellcolor{red!18}3547\,{\tiny (575)} & 0.606\,{\tiny (0.040)} & 3418\,{\tiny (487)} & 0.812\,{\tiny (0.051)} & 2941\,{\tiny (555)} & 0.974\,{\tiny (0.081)} & \underline{2611}\,{\tiny (413)} & 1.122\,{\tiny (0.038)} & 2079\,{\tiny (407)} & 1.203\,{\tiny (0.095)} \\
\cmidrule{2-14}
& LipsNet++ & \cellcolor{red!18}2344\,{\tiny (556)} & 0.993\,{\tiny (0.204)} & \cellcolor{red!18}2567\,{\tiny (448)} & 1.094\,{\tiny (0.204)} & \cellcolor{red!18}2521\,{\tiny (586)} & 1.183\,{\tiny (0.238)} & \cellcolor{red!18}2425\,{\tiny (604)} & 1.300\,{\tiny (0.218)} & 2289\,{\tiny (550)} & 1.382\,{\tiny (0.205)} & 2109\,{\tiny (505)} & 1.455\,{\tiny (0.191)} \\
\cmidrule{2-14}
& \methodname{} & \textbf{4071}\,{\tiny (614)} & \textbf{0.407}\,{\tiny (0.032)} & \textbf{3939}\,{\tiny (654)} & \textbf{0.519}\,{\tiny (0.040)} & \underline{3588}\,{\tiny (642)} & \textbf{0.690}\,{\tiny (0.060)} & \textbf{3384}\,{\tiny (705)} & \underline{0.870}\,{\tiny (0.075)} & \textbf{2774}\,{\tiny (760)} & \underline{0.987}\,{\tiny (0.103)} & \textbf{2427}\,{\tiny (736)} & \underline{1.116}\,{\tiny (0.132)} \\
\cmidrule[1pt]{2-14}
& SAC & 4189\,{\tiny (676)} & 1.268\,{\tiny (0.126)} & 4072\,{\tiny (555)} & 1.581\,{\tiny (0.134)} & 4032\,{\tiny (610)} & 1.982\,{\tiny (0.186)} & 3726\,{\tiny (422)} & 2.282\,{\tiny (0.196)} & 3427\,{\tiny (777)} & 2.513\,{\tiny (0.375)} & 2902\,{\tiny (751)} & 2.591\,{\tiny (0.355)} \\
\cmidrule{2-14}
& CAPS & \underline{4651}\,{\tiny (357)} & 1.178\,{\tiny (0.101)} & \underline{4567}\,{\tiny (298)} & 1.439\,{\tiny (0.098)} & \underline{4468}\,{\tiny (479)} & 1.789\,{\tiny (0.077)} & \underline{4360}\,{\tiny (698)} & 2.090\,{\tiny (0.153)} & 3937\,{\tiny (852)} & 2.262\,{\tiny (0.202)} & 3082\,{\tiny (1070)} & 2.258\,{\tiny (0.298)} \\
\cmidrule{2-14}
& L2C2 & 4413\,{\tiny (405)} & 1.081\,{\tiny (0.191)} & 4460\,{\tiny (296)} & 1.372\,{\tiny (0.169)} & 4393\,{\tiny (390)} & 1.772\,{\tiny (0.126)} & 4285\,{\tiny (248)} & 2.130\,{\tiny (0.067)} & \textbf{4187}\,{\tiny (302)} & 2.434\,{\tiny (0.077)} & 3492\,{\tiny (630)} & 2.511\,{\tiny (0.193)} \\
\cmidrule{2-14}
& ASAP & 4466\,{\tiny (400)} & \underline{0.795}\,{\tiny (0.056)} & 4409\,{\tiny (374)} & \underline{1.084}\,{\tiny (0.053)} & 4398\,{\tiny (346)} & \underline{1.481}\,{\tiny (0.046)} & 4315\,{\tiny (439)} & \underline{1.833}\,{\tiny (0.082)} & 3933\,{\tiny (481)} & \underline{2.081}\,{\tiny (0.100)} & 2987\,{\tiny (852)} & \underline{2.091}\,{\tiny (0.210)} \\
\cmidrule{2-14}
& LipsNet++ & 4419\,{\tiny (197)} & 1.093\,{\tiny (0.118)} & 4441\,{\tiny (230)} & 1.271\,{\tiny (0.125)} & 4430\,{\tiny (183)} & 1.581\,{\tiny (0.123)} & 4208\,{\tiny (362)} & 1.870\,{\tiny (0.135)} & 3837\,{\tiny (695)} & 2.109\,{\tiny (0.195)} & \textbf{3692}\,{\tiny (677)} & 2.328\,{\tiny (0.234)} \\
\cmidrule{2-14}
& \methodname{} & \textbf{4714}\,{\tiny (314)} & \textbf{0.419}\,{\tiny (0.048)} & \textbf{4691}\,{\tiny (321)} & \textbf{0.708}\,{\tiny (0.034)} & \textbf{4673}\,{\tiny (323)} & \textbf{1.078}\,{\tiny (0.044)} & \textbf{4531}\,{\tiny (404)} & \textbf{1.392}\,{\tiny (0.055)} & \underline{4184}\,{\tiny (522)} & \textbf{1.630}\,{\tiny (0.098)} & \underline{3583}\,{\tiny (934)} & \textbf{1.710}\,{\tiny (0.187)} \\
}

%% file: figures/ramp_sensitivity_table.tex
\begin{table}[t]
\centering
\caption{Effect of ramp-up on policy sensitivity (PPO).}
\label{tab:ramp_sensitivity}
\scriptsize
\setlength{\tabcolsep}{1.5pt}
\renewcommand{\arraystretch}{1.1}
\setlength{\aboverulesep}{0.25ex}
\setlength{\belowrulesep}{0.25ex}
\resizebox{\columnwidth}{!}{%
\begin{tabular}{@{}llcccccc@{}}
\toprule
& & \multicolumn{6}{c}{Policy sensitivity ($\times 10^{-4}$) $\downarrow$} \\
\cmidrule(lr){3-8}
Task & Method & $\sigma=0.05$ & $0.10$ & $0.15$ & $0.20$ & $0.25$ & $0.30$ \\
\midrule
\multirow{2}{*}{Ant} & Fixed-middle & 0.06\,{\tiny (0.01)} & 0.25\,{\tiny (0.05)} & 0.56\,{\tiny (0.11)} & 0.99\,{\tiny (0.20)} & 1.54\,{\tiny (0.30)} & 2.20\,{\tiny (0.43)} \\
 & CATS & \textbf{0.05}\,{\tiny (0.01)} & \textbf{0.19}\,{\tiny (0.03)} & \textbf{0.42}\,{\tiny (0.06)} & \textbf{0.75}\,{\tiny (0.11)} & \textbf{1.16}\,{\tiny (0.17)} & \textbf{1.65}\,{\tiny (0.24)} \\
\midrule
\multirow{2}{*}{Hopper} & Fixed-middle & 1.29\,{\tiny (0.20)} & 5.15\,{\tiny (0.81)} & 11.55\,{\tiny (1.81)} & 20.41\,{\tiny (3.16)} & 31.64\,{\tiny (4.85)} & 45.17\,{\tiny (6.84)} \\
 & CATS & \textbf{1.04}\,{\tiny (0.11)} & \textbf{4.16}\,{\tiny (0.45)} & \textbf{9.32}\,{\tiny (1.01)} & \textbf{16.47}\,{\tiny (1.82)} & \textbf{25.56}\,{\tiny (2.88)} & \textbf{36.49}\,{\tiny (4.20)} \\
\midrule
\multirow{2}{*}{Walker} & Fixed-middle & 4.23\,{\tiny (0.43)} & 16.88\,{\tiny (1.69)} & 37.86\,{\tiny (3.78)} & 67.04\,{\tiny (6.66)} & 104.22\,{\tiny (10.28)} & 149.17\,{\tiny (14.59)} \\
 & CATS & \textbf{4.05}\,{\tiny (0.46)} & \textbf{16.18}\,{\tiny (1.82)} & \textbf{36.31}\,{\tiny (4.06)} & \textbf{64.29}\,{\tiny (7.13)} & \textbf{99.94}\,{\tiny (10.98)} & \textbf{143.02}\,{\tiny (15.52)} \\
\bottomrule
\end{tabular}%
}
\vspace{-10pt}
\end{table}